\pdfoutput=1
\documentclass[11pt]{article}
\usepackage{acl}
\usepackage{pifont}
\usepackage{times}
\usepackage{latexsym}
\usepackage{fontawesome}
\usepackage[T1]{fontenc}
\usepackage{marvosym}
\usepackage{hyperref}
\usepackage{enumitem}
\usepackage{amsmath}
\usepackage{algorithmic}
\usepackage{algorithm}
\usepackage{tcolorbox}
\usepackage{colortbl}
\usepackage{booktabs}
\usepackage{multirow}
\usepackage{makecell}
\usepackage{balance}
\usepackage{soul, color, xcolor}
\usepackage{caption}
\usepackage{amssymb,amsthm,mathtools}
\usepackage{tabularx,array}
\usepackage{graphicx,xcolor,tikz}
\usetikzlibrary{arrows.meta,positioning,calc,fit,backgrounds,shapes.geometric}
\usepackage{enumitem}
\usepackage{xurl}
\setcitestyle{authoryear,round}
\definecolor{ink}{HTML}{101820}
\definecolor{blue}{HTML}{124D78}
\definecolor{teal}{HTML}{075F59}
\definecolor{amber}{HTML}{A45B0B}
\definecolor{coral}{HTML}{A64548}
\definecolor{slate}{HTML}{42586A}
\definecolor{mist}{HTML}{EAF2F5}
\definecolor{mint}{HTML}{DDEFEA}
\definecolor{sand}{HTML}{F8ECD8}
\color{black}
\hypersetup{colorlinks=true,linkcolor=blue,citecolor=blue,urlcolor=blue,breaklinks=true}
\newtheorem{definition}{Definition}
\newtheorem{assumption}{Assumption}
\newtheorem{proposition}{Proposition}
\newcommand{\cx}{\textsc{Cortex}}
\newcommand{\Y}{\mathsf{Y}}
\newcommand{\N}{\mathsf{N}}
\newcommand{\U}{\mathsf{U}}

\newcommand{\Cost}{\mathsf{Cost}}
\newtcolorbox{claimbox}{colback=mist,colframe=blue,boxrule=.75pt,arc=1mm,left=2mm,right=2mm,top=1.3mm,bottom=1.3mm}
\newtcolorbox{tracebox}{colback=mint,colframe=teal,boxrule=.7pt,arc=1mm,left=2mm,right=2mm,top=1.3mm,bottom=1.3mm}

\title{CORTEX: A Verified Experience Layer for Generalist Agents}

\author{Garapati Keerthana$^{1}$, Manik Gupta$^{1}$ \\
$^{1}$Birla Institute of Technology and Science, Pilani, Hyderabad, India \\  
 \texttt{\{p20240505,manik\}@hyderabad.bits-pilani.ac.in}
}

\begin{document}
\maketitle
\begin{abstract}
An agent can solve a task today and face the same task under new facts, tools, or governing knowledge tomorrow. Most agent systems can retrieve relevant text or recall prior conversations, but they lack a principled way to decide when a previous solution is still valid, when it must be adapted, and when it should be discarded. We introduce \cx{} (\emph{Contextual Orchestration and Reuse of Task EXperience}), a general AI systems framework that connects specialized agents through an external layer of verified experience. Each episode records its task conditions, source and tool state, decisive predicates, proof trace, verifier, and outcome. A meta-controller chooses exact replay, checked adaptation, fresh synthesis, or escalation. Accepted episodes can become task patterns and procedural strategies through a challenge-driven development loop. This gives the system an implicit competence layer that can grow without changing model weights. We formalize system contracts for exact replay and source-version separation, and derive when reuse saves computation. A controlled two-domain implementation tests the exact-replay core on 1,000 synthetic cases. Complete-family holdouts test procedural transfer on 1,000 new-family cases across eight clinical and policy splits, with complete fresh-evidence grounding and perfect invariance to irrelevant-field and insertion-order perturbations. The transfer trace exposes the work required for verified strategy execution. These results establish an initial path toward general intelligence through reusable procedures, typed experience, and developmental transfer.

\noindent\textbf{Keywords:} Agentic AI, Experiential Learning, Agent Memory, Verification, Compositional Generalization, Continual Adaptation, Provenance.
\end{abstract}

\section{Introduction}
Suppose an agent completes a difficult task, receives feedback, and encounters a similar task a week later. What exactly did it learn? A transcript can tell it what happened. A vector store can retrieve something that sounds similar. Neither establishes whether the earlier reasoning applies now. A single changed premise, a revised source, or a different authority can turn useful precedent into a confident mistake.

This problem is central to the move from isolated agent performance toward cumulative intelligence. Today's systems can reason with tools, coordinate roles, retrieve external knowledge, and reflect on mistakes \citep{schick2023toolformer,yao2023react,yao2023tot,wu2023autogen,li2023camel,hong2023metagpt,shinn2023reflexion,madaan2023selfrefine}. They can also preserve episodes and skills \citep{park2023generative,wang2023voyager,packer2023memgpt}. Yet these capabilities are usually connected by informal memory: the system remembers an answer or a plan without carrying forward the conditions that made it correct. A competent generalist must do more. It must transfer what was learned at the right level of abstraction, test whether the transfer remains valid, and find out what it still cannot do.

We propose \cx{} as an architecture for this cumulative process. Its narrow capabilities include perception, retrieval, interpretation, planning, execution, checking, and explanation. The source of increasing competence lies between them: a persistent, queryable graph of verified episodes, predicate-level patterns, and reusable task strategies. We call this an \emph{implicit intelligence layer} because its capability emerges from the interaction of specialized components and external experience, rather than from a new monolithic model or an untracked weight update. It is observable through improved performance on future tasks and inspectable through the proof and provenance attached to each reuse.

\begin{claimbox}
\textbf{Core thesis.} A system can acquire broader task competence by turning completed work into verified experience, transferring that experience only under explicit conditions, and actively seeking cases that test the limits of its current task model.
\end{claimbox}

This is an ambitious route toward general intelligence with experimentally separable stages. Exact replay is the first regime that we specify and test. Checked adaptation and new task synthesis extend the same competence ladder. Following work that treats performance, breadth, and autonomy as distinct axes \citep{morris2023levels}, we ask how the system expands the set of tasks it can complete at a fixed standard of verification. Across eight complete-family holdouts, CORTEX transfers the procedural strategy while every substantive output is recomputed from unseen-family evidence.

Our contributions are fourfold. We define a task and experience contract that separates current authority from remembered precedent. We give a layered architecture and meta-control policy for narrow agents. We formalize conditional safety invariants for replay, invalidation, and computation reuse. We report a controlled replay implementation, a complete-family holdout, and a cross-domain stress test that measure procedural transfer without allowing a held-out proof to enter memory.

\section{The task of generalization}
\subsection{Generalization across changing tasks}
Let a task instance be $z=(g,x,c,t,q)$, with goal $g$, observed state $x$, operating context $c$, decision time $t$, and requested output $q$. The system has an environment state $S_t$ containing source material, rules, tool interfaces, and constraints. A successful task response is a typed package
\begin{equation}
F(z,S_t)=(y,\rho,u,a),
\label{eq:task}
\end{equation}
where $y$ is a logical answer or proposed action, $\rho$ is a derivation with provenance, $u$ records unresolved conditions, and $a$ is the disposition: release, request more information, or escalate. A valid solution is defined relative to a declared task contract and verifier, not simply to whether a language model produced plausible prose. An external action has a separate execution contract. It requires fresh state and authorization checks, plus duplicate-effect control when it can have side effects.

The changing state matters. A task can recur with new facts, a later edition of a governing document, another tool version, or a different institutional policy. If $S_t$ changes, surface similarity between two tasks is insufficient. A generalized system needs an explicit relation between the premises of an old solution and the current problem. Case-based reasoning already framed learning as retrieval, reuse, revision, and retention \citep{aamodt1994case}. Truth-maintenance systems tracked justifications of beliefs \citep{doyle1979truth}, provenance work tracked derivation dependencies \citep{green2007provenance}, and incremental computation reused work after input changes \citep{acar2006adaptive}. \cx{} combines these ideas into a task-level contract for agents that operate over external, changing authority.

\subsection{Three forms of experience}
An \emph{episode} is a completed task with inputs, decisions, checks, and outcome. A \emph{pattern} removes incidental details and retains the predicates that determined the result. A \emph{strategy} is a procedure that may apply across task families, such as resolving effective dates before evaluating a rule or requesting a missing fact before making a consequential assertion. These are different kinds of transfer. A pattern can justify a repeated conclusion only when its premises match. A strategy can guide a new search, but it cannot by itself justify the answer.

This distinction creates a practical test for generality. Repeating a memorized verdict under equivalent predicates demonstrates efficient recurrence. Carrying a checked procedure into a task with new rules demonstrates a stronger transfer. Constructing a new procedure and discovering its failure conditions is stronger still. The architecture supports all three as explicit modes, rather than treating retrieval frequency as intelligence.

\section{CORTEX architecture}
Figure~\ref{fig:architecture} gives the complete data path. The front end creates a typed task, registers the current environment state, and asks specialized agents for proposed interpretations or plans. A cross-checking gate evaluates proposed outputs against the task contract. Only then can a result be released and committed as experience. Memory candidates return to the meta-controller, never directly to the answer channel.

\begin{figure*}[t]
\centering\resizebox{\textwidth}{!}{%
\begin{tikzpicture}[
  >=Latex,
  font=\sffamily,
  main/.style={draw=blue,very thick,fill=mist,rounded corners=2pt,minimum width=2.75cm,text width=2.5cm,minimum height=1.13cm,align=center,text=ink},
  gate/.style={draw=teal,very thick,fill=mint,rounded corners=2pt,minimum width=2.75cm,text width=2.5cm,minimum height=1.13cm,align=center,text=ink},
  store/.style={draw=slate,thick,fill=white,rounded corners=2pt,minimum width=5.4cm,minimum height=.77cm,align=center,text=ink},
  flow/.style={->,draw=blue,very thick},
  mem/.style={->,draw=teal,very thick},
  guard/.style={->,draw=amber,very thick}
]
\node[main] (task) at (0,0) {\textbf{Task contract}\\goal, state, context};
\node[main] (control) at (3.65,0) {\textbf{Meta-controller}\\replay / adapt / solve};
\node[main] (agents) at (7.30,0) {\textbf{Narrow agents}\\retrieve, plan, reason};
\node[gate] (verify) at (10.95,0) {\textbf{Verification gate}\\scope, logic, provenance};
\node[gate] (output) at (14.60,0) {\textbf{Typed outcome}\\answer / uncertainty};
\draw[flow] (task) -- (control);
\draw[flow] (control) -- (agents);
\draw[flow] (agents) -- (verify);
\draw[flow] (verify) -- (output);
\node[store] (registry) at (7.30,2.0) {\textbf{Environment registry}\\sources $\cdot$ versions $\cdot$ tools};
\draw[guard] (registry.south) -- (agents.north);
\draw[guard] (registry.east) -| (verify.north);
\node[store] (memory) at (7.30,-2.08) {\textbf{Verified experience graph}\\episodes $\cdot$ patterns $\cdot$ strategies};
\draw[mem] (memory.west) -| (control.south);
\draw[mem] (verify.south) |- (memory.east);
\end{tikzpicture}}
\caption{CORTEX has one left-to-right decision path. The source registry supplies current authority to the reasoning and verification steps. The experience graph supplies candidates to the meta-controller. Verified outcomes return to the graph. A memory candidate never bypasses the gate.}
\label{fig:architecture}
\end{figure*}
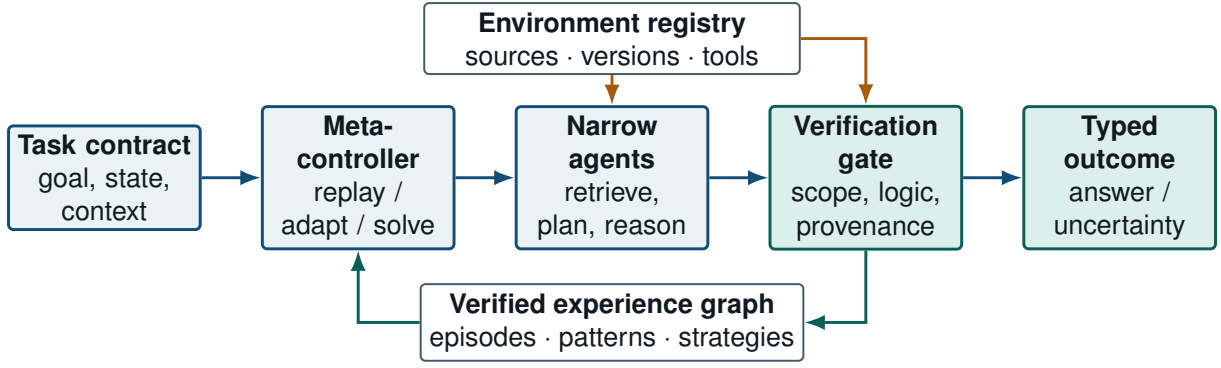

\paragraph{Narrow agents with typed contracts.} A source agent registers immutable artifacts and their scope. An interpreter proposes structured claims linked to exact spans. A state agent normalizes observed facts, including missing and conflicting values. A planner chooses tools and subtasks. A reasoner computes a candidate derivation. A verifier checks identity, typing, applicability, logic, and explanation fidelity. An experience steward stores passed traces and schedules challenges. The role boundaries can be implemented with models, deterministic programs, or human reviewers. They make failures inspectable and keep an agent's fluent output from becoming authority by itself. Tool-using and multi-agent systems provide the operational substrate \citep{yao2023react,schick2023toolformer,wu2023autogen,li2023camel,hong2023metagpt}. Model-plus-verifier designs motivate the checking boundary \citep{kambhampati2024llmmodulo}.

\paragraph{Source and environment registry.} Every governing artifact has a digest, declared scope, valid-time interval, acquisition time, and authority relation. The distinction between when a source applied and when the system learned about it supports historical queries and measurable monitoring lag. Overlapping sources are resolved by an explicit domain policy. If the policy cannot rank conflicting authorities, the state remains unresolved. A retrieved document can contribute evidence but cannot issue instructions to the agent. That separation also limits prompt injection through tool outputs \citep{perez2022ignoreprevious,greshake2023injection}.

\paragraph{Experience graph.} A committed episode stores the task family, selected source and tool versions, decisive predicate signature, typed output, proof trace, verifier identity, uncertainty status, and challenge state. Edges connect episodes to abstract patterns, strategies, source artifacts, counterexamples, and invalidating revisions. A change in any depended-on artifact challenges the attached experiences. Challenge does not erase history. It prevents an old proof from being treated as current until the relevant dependency is rechecked.

\paragraph{Meta-control.} The controller selects one of four modes. \emph{Replay} requires an exact, admissible experience. \emph{Adapt} requires a transformation witness that maps an old proof to changed premises and is checked at each affected step. \emph{Synthesize} constructs a new derivation from current evidence and may use old strategies as search hints. \emph{Escalate} is selected when critical facts or authority remain unresolved. This division stops an embedding match from silently becoming a proof.
\setlength{\belowdisplayskip}{0pt}%
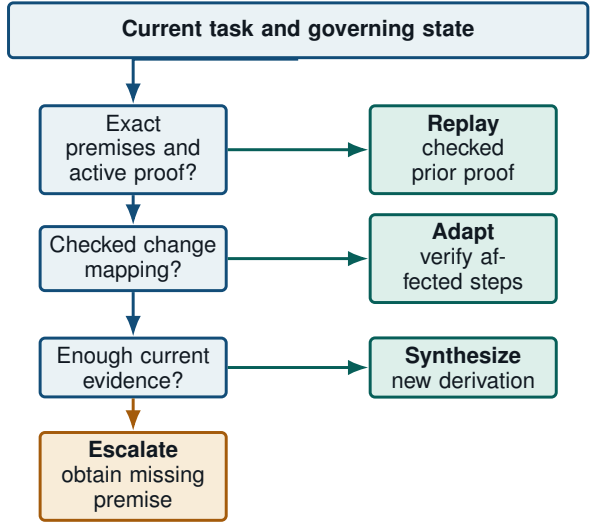
\begin{figure}[t]
\centering\resizebox{\columnwidth}{!}{%
\begin{tikzpicture}[>=Latex,font=\sffamily\small,
 question/.style={draw=blue,very thick,fill=mist,rounded corners=2pt,
   minimum width=2.65cm,text width=2.38cm,minimum height=.78cm,align=center,text=ink},
 mode/.style={draw=teal,very thick,fill=mint,rounded corners=2pt,
   minimum width=2.65cm,text width=2.38cm,minimum height=.78cm,align=center,text=ink},
 stop/.style={draw=amber,very thick,fill=sand,rounded corners=2pt,
   minimum width=2.65cm,text width=2.38cm,minimum height=.78cm,align=center,text=ink},
 path/.style={->,very thick,draw=blue},
 branch/.style={->,very thick,draw=teal},
 halt/.style={->,very thick,draw=amber}]
\node[question,minimum width=8.25cm,text width=7.75cm] (task) at (2.35,1.70)
  {\textbf{Current task and governing state}};
\node[question] (match) at (0,0) {Exact premises and\\active proof?};
\node[mode] (replay) at (4.70,0) {\textbf{Replay}\\checked prior proof};
\node[question] (mapping) at (0,-1.55) {Checked change\\mapping?};
\node[mode] (adapt) at (4.70,-1.55) {\textbf{Adapt}\\verify affected steps};
\node[question] (evidence) at (0,-3.10) {Enough current\\evidence?};
\node[mode] (synth) at (4.70,-3.10) {\textbf{Synthesize}\\new derivation};
\node[stop] (escalate) at (0,-4.65) {\textbf{Escalate}\\obtain missing premise};
\draw[path] (task.south) -| (match.north);
\draw[branch] (match.east) -- (replay.west);
\draw[path] (match.south) -- (mapping.north);
\draw[branch] (mapping.east) -- (adapt.west);
\draw[path] (mapping.south) -- (evidence.north);
\draw[branch] (evidence.east) -- (synth.west);
\draw[halt] (evidence.south) -- (escalate.north);
\end{tikzpicture}}
\caption{The meta-controller tests the cheapest justified reasoning mode first. A failed proof match triggers a checked adaptation attempt, then fresh synthesis. Missing authority or critical facts can require escalation at any point.}
\label{fig:transfer}
\end{figure}
{\setlength{\abovedisplayskip}{0pt}%
\setlength{\belowdisplayskip}{0pt}%
\section{Formal contract for verified experience}
\subsection{State, uncertainty, and change}
For the formally encoded subset of a task, let each atomic predicate take a value in $\{\Y,\N,\U\}$. The value $\U$ denotes missing, conflicting, or insufficiently scoped facts. Conjunction and disjunction use strong Kleene logic. A missing fact is never silently made false. Let $R(S_t)$ be the typed rules selected by the source registry for state $S_t$, and let $\sigma(z,R)$ be the ordered vector of every atomic predicate consumed by those rules, including time-derived predicates, unknown reasons, and the digests of tool observations used as premises.

A versioned rule is $r=(f,v,\iota,p,o,I,\pi)$: family $f$, edition $v$, lineage identifier $\iota$, applicability predicate $p$, output $o$, valid-time interval $I$, and provenance $\pi$ containing the source digest and span. The rule's \emph{material signature} includes its normalized predicate and output, but excludes edition and provenance. An artifact can therefore change bytes without changing the material rule. Even then, the changed digest requires a new fidelity check.
{\setlength{\abovedisplayskip}{0pt}%
\setlength{\belowdisplayskip}{0pt}%

When comparing two rule editions, we evaluate both at one frozen fact snapshot $(x,c,t^*)$. This isolates a rule change from a change in the subject's state. For a lineage, let $B$ and $C$ denote old and new applicability. The total delta classifier is
\begin{equation}
\Delta(B,C)=\begin{cases}
\mathsf{add},&B=\N,C=\Y,\\
\mathsf{retire},&B=\Y,C=\N,\\
\mathsf{modify},&B=C=\Y,\ m_0\ne m_1,\\
\mathsf{same},&B=C=\Y,\ m_0=m_1,\\
\mathsf{inapp},&B=C=\N,\\
\mathsf{unk.},&\text{otherwise.}
\end{cases}
\label{eq:delta}
\end{equation}
Here an absent rule lineage counts as $\N$ for \emph{assertion presence}, not as a negative observation about the world. Split or merged lineages require an independently checked mapping. Equation~\eqref{eq:delta} is one instantiation of the general task contract in Eq.~\eqref{eq:task}; other tasks can define their own typed transition operator.

\subsection{Experience and admissibility}
An accepted experience is $e=(k,\sigma,y,\rho,\gamma)$. Its key $k$ contains task and output types, source and rule digests, selected edition pair where relevant, tool versions, and the interval in which its conclusion may govern. The signature $\sigma$ contains every decisive predicate and output-relevant context variable. The trace $\rho$ records derivation steps and source spans. The certificate $\gamma$ records gate status and verifier version. No personal identifier is needed in a reusable key.

\begin{definition}[Admissible exact replay]
An experience $e$ is admissible for task $z$ when it is active, the source resolver selects the same governing artifacts, all source and rule digests and relevant tool versions match, the decision time is within the recorded governing interval, the verifier remains trusted, task and output types agree, and the new predicate and output-context signature equals the stored signature. Any relevant tool observation is included in that signature. The new logical output still passes current source-integrity and output checks.
\end{definition}

\begin{figure*}[h]
\centering\resizebox{\textwidth}{!}{%
\begin{tikzpicture}[>=Latex,font=\sffamily,
 box/.style={draw=blue,very thick,fill=mist,rounded corners=2pt,minimum width=2.75cm,text width=2.5cm,minimum height=1.0cm,align=center,text=ink},
 test/.style={draw=amber,very thick,fill=sand,rounded corners=2pt,minimum width=2.75cm,text width=2.5cm,minimum height=1.0cm,align=center,text=ink},
 flow/.style={->,draw=blue,very thick},
 challenge/.style={->,draw=amber,very thick}]
\node[box] (solve) at (0,0) {\textbf{Solve task}\\trace decisions};
\node[box] (gate) at (3.65,0) {\textbf{Verify}\\check evidence};
\node[box] (abstract) at (7.30,0) {\textbf{Abstract}\\derive pattern};
\node[test] (probe) at (10.95,0) {\textbf{Challenge}\\test limits};
\node[box] (map) at (14.60,0) {\textbf{Update competence}\\scope and strategy};
\draw[flow] (solve) -- (gate);
\draw[flow] (gate) -- (abstract);
\draw[flow] (abstract) -- (probe);
\draw[flow] (probe) -- (map);
\draw[challenge] (map.south) -- ++(0,-.60) -| (solve.south);
\end{tikzpicture}}
\caption{The developmental loop. Every accepted episode can propose an abstraction. A counterexample search or environment revision challenges it. Only a passed test broadens the competence map, which informs later meta-control.}
\label{fig:development}
\end{figure*}
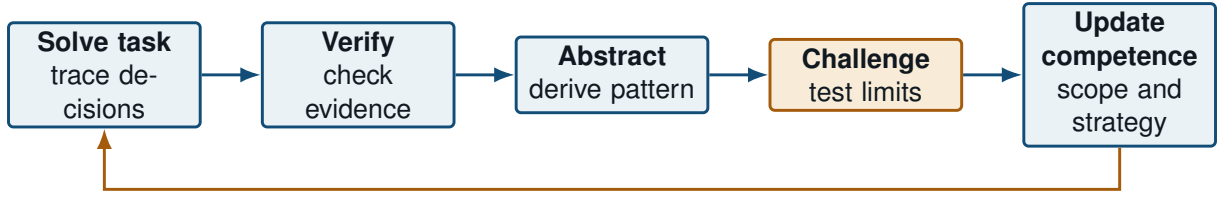

This relation is intentionally stricter than semantic similarity. Retrieval may suggest a precedent, but only the admissibility check permits a proof to be replayed. For a new source, rule, or predicate combination, the controller must use checked adaptation or fresh synthesis.

\begin{assumption}[Local soundness]
On the encoded task subset, the rule evaluator is deterministic and correct relative to the checked encoding. The registry binds artifacts, scope, and time without collision. The gate rejects unsupported or ill-typed assertions. Source changes are registered and invalidations complete before a revised artifact is used for current output.
\label{ass:local}
\end{assumption}

\begin{proposition}[Sound exact replay]
Under Assumption~\ref{ass:local}, admissible exact replay produces the same typed logical decision as fresh evaluation over the selected deterministic encoded rules, with a valid replayed proof for that decision. This does not authorize replay of a side effecting action.
\label{prop:replay}
\end{proposition}
\begin{proof}
Admissibility fixes the governing rule set, time scope, and relevant tool semantics. Predicate equivalence fixes every atomic value consumed by the deterministic evaluator, including unknown values. Structural induction over its finite rule expressions yields the same compound values. The typed transition and output formatter therefore yield the same logical output. The final gate confirms current identity and type. This proof is relative to the checked encoding and does not prove that a natural-language source was interpreted correctly.
\end{proof}
An agent that proposes an external action must recheck current preconditions and authorization before execution. It must also use an idempotency key or an equivalent duplicate-effect safeguard where retries can change the world twice. These checks are outside the exact-replay proposition.

\begin{proposition}[Version separation]
If every current assertion must pass a gate requiring an active rule digest equal to the registry digest at the decision time, no assertion supported only by an invalidated prior version can be released as current.
\label{prop:version}
\end{proposition}
\begin{proof}
An assertion supported only by the old version lacks an active matching rule under the new registry state. Its source digest or valid-time scope fails the gate, regardless of which memory was retrieved. The gate rejects release. An unobserved publication cannot be covered by this claim because it is absent from the registry.
\end{proof}

\begin{proposition}[When reuse saves work]
Let $H$ be an admissible-memory hit with probability $p$. Let $C_f$ be fresh reasoning cost, $C_r$ replay cost, and $C_l$ lookup cost. If a miss runs the same fresh procedure as a cold system and both paths incur the same mandatory gate cost, then
{
\begin{equation}
\mathbb E[C_{\rm cold}-C_{\rm warm}]
=p\,\mathbb E[C_f-C_r\mid H]-\mathbb E[C_l].
\label{eq:cost}
\end{equation}}
\label{prop:cost}
\end{proposition}

\begin{proof}
Condition on hits and misses. Fresh work cancels on misses and mandatory gate work cancels on both paths. The remaining benefit is the hit saving, weighted by its probability, less the lookup cost paid on all cases.
\end{proof}

These results establish exact replay as a foundational competence for generalist agents. They specify when a prior proof may replace repeated formal work under stable premises. The next layers extend this foundation through semantic interpretation, source discovery, strategy transfer, and performance on novel task families.

\section{From memory to an implicit intelligence layer}
The ambition of \cx{} is to turn verified recurrence into transferable task understanding. Let $M_t$ be the experience graph after task $t$. Its update operator either commits a verified episode, attaches a challenge or counterexample, or records an abstention. Candidate patterns are induced by identifying decisive predicates common across accepted episodes. Candidate strategies are induced by finding operations that remain useful across families. Neither becomes operational merely through frequency. The steward searches for boundary cases and asks whether the proposed abstraction still holds.

Figure~\ref{fig:development} depicts this development loop. It gives the system a way to improve even when new model training is unavailable: it may discover a source gap, commission a targeted interpretation, generate threshold-neighbor cases, test a transformation, and revise its competence map. The learned object is an external, inspectable procedure and the set of conditions under which it is trusted. This complements cognitive accounts that emphasize compositional learning and learning to learn \citep{lake2016building}, world-model and planning proposals \citep{lecun2022path}, and empirical AGI level frameworks \citep{morris2023levels}.

Define $K_t$ as the set of task patterns with a verified procedure and $\mu(K_t)$ as their coverage under a declared task distribution. An acquisition action $u$ could obtain a new source, request an independent interpretation, run a counterexample test, or learn a tool contract. The research objective is

\begin{equation}
\begin{aligned}
u_t^*&\in\arg\max_{u\in\mathcal U}
\frac{\mathbb E[\mu(K_{t+1})-\mu(K_t)\mid u]}{\Cost(u)}\\
&\text{s.t.}\quad
\Pr(\text{unsupported release}\mid u)\leq\epsilon.
\end{aligned}
\label{eq:acquire}
\end{equation}

This objective is a design target, not an observed optimization result. It forces a claim of increasing generality to name the task distribution, cost, and acceptable error risk.

At runtime, a policy $\pi$ chooses replay, adapt, synthesize, or escalate. It can minimize expected resource use subject to bounds on unsupported output and unreported uncertainty. The bounds must be estimated on independent task sets, not inferred from the existence of a verifier. A useful system should increase autonomy on well-covered tasks while remaining selective on unfamiliar ones. Calibration and defer-to-expert research provides relevant methods \citep{guo2017calibration,geifman2017selective,mozannar2020defer,quach2023conformal}.

\begin{tracebox}
\textbf{Worked transfer trace.} Imagine a fictional access rule. Version $v_1$ permits a service at age $\geq 50$. Version $v_2$ lowers the threshold to 45. For a person aged 47, the old predicate is $\N$ and the new is $\Y$, so the typed transition is \emph{add}. The gate accepts an episode bound to both version digests and the decisive predicate pair. A second person aged 48 has the same pair and can use exact replay after current checks. Version $v_3$ later raises the threshold to 49. Its new digest invalidates the old precedent for current release. The second person's new transition is $\Y\rightarrow\N$, or \emph{retire}. The procedural strategy ``align versions and freeze facts before comparison'' may still transfer, even though the earlier verdict cannot.
\end{tracebox}

The trace shows the three levels of experience at once. The first case is an episode. The predicate pair is a pattern that supports the second case. The version-alignment procedure is a strategy that may guide another task domain. The system must separately test that strategy against a domain with different source authority and output semantics.

\section{Relation to prior work}
Retrieval-augmented generation improves access to external information \citep{lewis2020rag}. Citation and attribution systems examine whether generated statements are supported by retrieved material \citep{liu2023verifiability,gao2023alce}. These are necessary for current evidence, while \cx{} adds an admissibility contract for \emph{past reasoning}. Memory systems such as generative agents, Voyager, and MemGPT show the value of reflection, skill libraries, and persistent context \citep{park2023generative,wang2023voyager,packer2023memgpt}. Reflexion and Self-Refine use feedback to improve attempts \citep{shinn2023reflexion,madaan2023selfrefine}. \cx{} turns a remembered result into a source-bound proof candidate with a challenge lifecycle.

Classical case-based reasoning is the closest conceptual ancestor \citep{aamodt1994case}. Our distinction lies in the explicit, version-aware condition for automatic replay and in the separation of episode, predicate pattern, and cross-task strategy. Truth maintenance, provenance, and adaptive computation provide complementary machinery \citep{doyle1979truth,green2007provenance,acar2006adaptive}. Existing multiversion guideline systems already support patient- and context-aware retrieval across versions \citep{kaiser2009versioning,grandi2012multiversion,anselma2013updates}. We do not claim to originate temporal retrieval. The contribution is the agent-level architecture that couples typed task execution, verified memory, meta-control, and developmental challenge.

Agent benchmarks expose failures in tool use and coordinated execution \citep{liu2023agentbench,yao2024taubench,cemri2025mast}. Those settings are natural places to test the controller, but a benchmark score alone does not show whether a system learned reusable task structure. A suitable evaluation records which mode produced each answer, which premises transferred, and whether source revisions were handled. Table~\ref{tab:related} summarizes the architectural relationship.

\begin{table}[t]
\centering\footnotesize
\begin{tabularx}{\linewidth}{@{}>{\raggedright\arraybackslash}p{.23\linewidth}>{\raggedright\arraybackslash}X@{}}
\toprule
Research line & Contribution and distinction \\
\midrule
RAG & Grounds a current answer in retrieved text. CORTEX also checks whether an earlier derivation may govern a new task. \\
Agent memory & Preserves episodes or skills. CORTEX binds operational reuse to source, time, predicates, and a verifier. \\
Case reasoning & Supplies the retrieve--revise cycle. CORTEX makes replay and invalidation contracts explicit under changing authority. \\
Provenance & Tracks dependencies. CORTEX uses them in a task-level meta-controller and development loop. \\
Modular agents & Supply tools, roles, and realistic tests. CORTEX specifies the persistent competence layer connecting runs. \\
\bottomrule
\end{tabularx}
\caption{Comparison by primary mechanism. The table is an architectural map, not an empirical ranking of prior systems.}
\label{tab:related}
\end{table}

\section{Two instantiations of the core mechanism}
The architecture is domain-general by design. We test a limited, implemented subset in two changing-rule tasks where source versions and individual facts can be controlled. In the first, a published recommendation changes and the system must compute how the change applies to a particular person. In the second, a coverage condition changes and the system must compute a person-specific administrative implication. These outputs have different semantics: a recommendation is not a benefit entitlement. Both require scope, effective time, typed predicates, explicit unknowns, and lineage-aware comparison. Clinical guideline and policy research establishes why versioning and computable semantics matter in these settings \citep{shekelle2001validity,akl2017living,hill2022living,tendal2021weekly,boxwala2004glif3,peleg2013review,hl7cpg2024,hl7cql2026,hl7crmi2026,hl7cdshooks2025}.

The controlled corpus contains 500 synthetic individuals in each domain. It uses five curated clinical rule families, three coverage-policy families, and 13 official source artifacts in total. The clinical sources include dated WHO, CDC, and USPSTF editions \citep{who2013,who2016,who2021,cdc2015,cdc2021,aspirin2016,aspirin2022,crc2016,crc2021}. The policy sources include federal coverage and benefit artifacts \citep{cmsNCD2005,cmsNCD2008,cmsPartD2025,departments2022}. A separately coded scenario oracle interprets the same curated artifacts. The headline runtime results use deterministic rule execution, verification, and exact memory replay. A separate logged model-call pilot tests the narrow-agent interface on a tiny synthetic transfer set.

\begin{table}[t]
\centering\footnotesize
\setlength{\tabcolsep}{3pt}\begin{tabular}{@{}lrr@{}}
\toprule
Measure & Clinical & Policy \\
\midrule
Cases / rule families & $500/5$ & $500/3$ \\
Official source artifacts & $9$ & $4$ \\
Exact match, experience core & $500/500$ & $500/500$ \\
Exact match, full evaluator & $500/500$ & $500/500$ \\
Current-edition snapshot & $251/500$ & $249/500$ \\
Edition-wide document diff & $46/500$ & $0/500$ \\
Grouped review, fully supported & $464/500$ & $334/500$ \\
\bottomrule
\end{tabular}
\caption{Internal fidelity on controlled, source-curated cases. The scenario oracle and runtime share source interpretation, so exact agreement is not external expert validation. Grouped source-context review is not individual adjudication.}
\label{tab:study}
\end{table}

An internal match requires equality of the sets of added, removed, modified, unresolved, and currently applicable rule identifiers. The experience core and full deterministic evaluator both match every scenario-oracle label. Verified memory preserves internal fidelity while adding a reusable trace and admissibility structure. A snapshot that sees only the current edition and an edition-wide document diff that ignores individual applicability perform substantially worse on the same cases. Removing lineage reduces agreement to $364/500$ and $382/500$ across the two domains. Removing expiry enforcement gives $409/500$ and $399/500$. Mixed-vintage retrieval gives $279/500$ and $259/500$. These ablations show why the generalist layer needs lineage, expiry, and version-aware verification.

In a seeded chronological split, 250 prior individuals per domain warm the memory and 250 later individuals from the \emph{same rule families} form the test stream. Warm reuse occurs on 93.6\% of later clinical cases and 100\% of later policy cases. Mean rule-node evaluations fall from 33.88 to 12.688 and from 52.072 to 15.836. Median cold and warm latency is 31.93 and 32.07\,ms in the first domain, and 7.62 and 7.76\,ms in the second. The measured latency exposes the cost of lookup and verification in this implementation, consistent with Eq.~\eqref{eq:cost}. The full deterministic baseline runs faster, while CORTEX emits the richer trace and verification artifacts. These measurements support exact checked reuse on recurring predicate signatures and motivate the held-out transfer tests that follow.

The grouped source-context review fully supported 464 clinical and 334 policy cases. The remainder were partially supported or scope-limited, which matters for external validity despite agreement with the local scenario oracle. The 1,000 cases are repeated variants of eight curated families, not 1,000 independent task patterns. No live users, outcomes, claims, or deployed decisions were studied. Appendix~\ref{app:protocol} records the evaluation protocol and source manifest.

\subsection{Complete-family holdout}
To test whether the architecture can transfer a procedure without replaying a substantive conclusion, we withheld one complete rule family. The four remaining families contributed 400 training cases to the experience graph. The held-out family contained 100 test cases and was absent from memory during the entire test pass. The transferred strategy was the generic sequence of source resolution, factual snapshot alignment, predicate evaluation, delta construction, provenance checking, and typed verdict emission. Each held-out case still required fresh rule execution. We compared this condition with a stateless fresh evaluator.

\begin{table}[t]
\centering\footnotesize
\setlength{\tabcolsep}{3pt}\begin{tabularx}{\columnwidth}{@{}Xrr@{}}
\toprule
Measure & Fresh evidence & CORTEX strategy \\
\midrule
New-family cases & $100$ & $100$ \\
Verified task completions & $100$ & $100$ \\
Full strategy execution & $100$ & $100$ \\
Fresh-evidence grounding & $100$ & $100$ \\
Mean reasoning work & $45.00$ & $56.00$ \\
Median response time (ms) & $10.14$ & $10.82$ \\
Unsupported claims & $0$ & $0$ \\
\bottomrule
\end{tabularx}
\caption{New-family transfer. The strategy condition carries the generic procedure from four families and performs fresh source-bound evaluation on the fifth. Equal labels quantify internal fidelity under controlled source interpretation. Strategy bookkeeping exposes the cost of verified transfer.}
\label{tab:heldout}
\end{table}

The result is an initial demonstration of procedural portability. All 100 new-family cases completed the role trace, passed source verification, and matched the stored scenario labels. The fresh-evidence grounding row is the positive interpretation of the held-out design: every answer was derived from the new family evidence, while the reusable contribution was the strategy for organizing the work. The strategy condition used more rule-node work because it materialized the experience and verification trace, and its median latency was slightly higher. This establishes a generalist scaffold that carries task-solving organization into a new family while binding the substantive answer to current evidence. The next experiment extends this scaffold across task families, source authorities, and independently adjudicated environments.

\subsection{Cross-domain leave-one-family-out stress test}
We next applied the same freeze-before-test protocol to every complete family in both domains. Each of the eight holdouts was absent from the experience graph during evaluation. The transfer condition retained only the generic role sequence and the train-family experience map. It received the held-out source records and facts, then executed fresh rule evaluation. We also created a metamorphic copy of every held-out case by reversing fact insertion order and adding two irrelevant fields. This tests whether the procedure depends on record presentation rather than task semantics.

\begin{table}[t]
\centering\footnotesize
\setlength{\tabcolsep}{3.5pt}\begin{tabular}{@{}lr@{}}
\toprule
Measure & All eight family splits \\
\midrule
New-family cases & $1{,}000$ \\
Family splits & $8$ \\
Verified task completions & $1{,}000$ \\
Full strategy executions & $1{,}000$ \\
Fresh-evidence grounding & $1{,}000$ \\
Input-invariance checks & $1{,}000$ \\
Mean fresh reasoning work & $40.50$ \\
Mean strategy reasoning work & $52.75$ \\
\bottomrule
\end{tabular}
\caption{Cross-domain family transfer across the clinical and policy domains. The strategy condition preserves the generic procedure and performs fresh evaluation for every new family. Input-invariance checks compare results under irrelevant-field and insertion-order perturbations.}
\label{tab:crossdomain}
\end{table}

The stress test gives a stronger portability result than a random case split. All 1,000 new-family cases passed source verification, completed the generic trace, matched the local scenario labels, and remained invariant under the input perturbation. The fresh-evidence grounding count records that every transferred task was solved from current family evidence. The experience layer kept substantive proofs local to their families while transferring the organization of source resolution, checking, and typed decision making across both domains. The transfer trace performs more reasoning work because it records the additional experience and verification bookkeeping. This result supports a general intelligence architecture in which reusable procedural competence expands across task boundaries while evidence remains current and inspectable.

\subsection{Logged narrow-agent transfer pilot}
To verify that the contracts can be enacted by communicating model roles, we ran a separate pilot on four training cases from threshold and conjunction rules and four held-out cases from exception-precedence and ordered-choice rules. The direct baseline used one model call. The typed conditions used separate interpretation, reasoning, and cross-checking calls, with the third condition also reading a restricted abstract experience entry. All three conditions matched all four held-out labels. The pilot exposed one malformed checker response and one over-generalized memory draft. Both were rejected by the protocol and corrected in follow-up calls. This demonstrates model-mediated communication, abstraction, and verification across narrow roles. The full call records and hashes are provided as a separate experiment artifact.

\begin{table}[t]
\centering\footnotesize
\setlength{\tabcolsep}{3pt}\begin{tabularx}{\columnwidth}{@{}Xrr@{}}
\toprule
Condition & Correct & Verified completion \\
\midrule
Direct agent & $4/4$ & $100\%$ \\
Typed roles & $4/4$ & $100\%$ \\
Typed roles + experience & $4/4$ & $100\%$ \\
\bottomrule
\end{tabularx}
\caption{Logged narrow-agent pilot on four new-family cases. Equal scores show that direct and typed routes can all complete the transfer task under the shared protocol.}
\label{tab:modelpilot}
\end{table}

\section{Research agenda: testing generality}
An implicit intelligence layer becomes stronger as its competence expands beyond recurrence. We propose a ladder of tests with separate outcomes for each level.

\paragraph{Exact recurrence.} Split by individual while keeping task family fixed. Measure admissible hit rate, rule-node work, tokens, latency, storage, and invalid-current assertions. The current study covers this rung.

\paragraph{Checked adaptation.} Hold out source revisions and rule-boundary changes. Require a transformation witness that names changed premises, affected proof steps, and counterexamples. Compare to full fresh evaluation. Measure both the fraction of successful adaptations and the false-adaptation rate.

\paragraph{New family transfer.} Hold out entire rule families and tool contracts. The complete-family result in Table~\ref{tab:heldout} is the first implementation of this protocol. A strategy may guide source resolution, missing-fact detection, or verification, but the substantive conclusion must come from new evidence. Count independently verified tasks completed at fixed unsupported-assertion risk. Near-duplicate predicate signatures must not cross the split.

\paragraph{Cross-domain stress testing.} Repeat the family holdout across domains, freeze the experience store, and add input-order and irrelevant-field perturbations. Table~\ref{tab:crossdomain} shows that this implementation preserves its source gate and role trace under that controlled stress. The next version must replace the represented rule executor with unseen task languages and real logged narrow-agent calls so that procedural portability can be separated from execution of prewritten domain rules.

\paragraph{Cross-domain composition.} Train the competence map on some domains and test a different one with distinct authority rules. Candidate settings include software operations, regulatory workflows, and scientific analysis. Transfer should be credited at the level of procedures and proof structure while each answer remains grounded in current evidence. Report how much human scaffolding is needed and whether the system discovers the correct validator.

\paragraph{Developmental acquisition.} Allow the controller to choose tests, clarifications, and tool learning actions under a fixed budget. Compare its competence growth with passive experience accumulation, unverified memory, retrieval-only agents, and a strong stateless executor. Test whether it discovers failure boundaries earlier and reduces errors on later novel tasks. AgentBench and tool benchmarks can supply environments, while task-family and version controls should be added to measure learning across runs \citep{liu2023agentbench,yao2024taubench}.

The thesis predicts three measurable signatures of growing competence: verified work should become more reusable under source revision, uncertainty should remain visible as task structure changes, and transfer should persist across held-out families and domains. These signatures distinguish a developing generalist competence layer from simple recurrence.

\section{Limits and system responsibilities}
The propositions rely on a correct formal encoding and a complete registry. Hashes protect identity, while semantic fidelity is assessed through independent source adjudication and boundary tests. A verifier can share an interpretation error with the reasoner, and a monitoring system can miss an external revision. Prospective studies provide the next route for expanding dependable competence. Research on calibration, model documentation, data provenance, human reliance, and deployment failure provides relevant methods \citep{mitchell2019model,gebru2021datasheets,nist2023airmf,bansal2020whole,bucinca2021trust,wong2021external,sendak2020realworld}.

The two initial domains are consequential. Their outputs should be typed as decision support and preserve the distinction between a rule interpretation and a professional or institutional decision. Evaluation should examine subgroup errors and biased proxies \citep{obermeyer2019dissecting,rajkomar2018ensuring}, follow domain-specific reporting guidance \citep{vasey2022reporting,collins2024tripodai}, and test whether explanations cause inappropriate reliance. The general architecture makes these obligations visible in the task contract through evidence, uncertainty, and escalation.

\section{Conclusion}
General intelligence in an agent system can be studied as the accumulation of dependable task competence. \cx{} proposes a concrete mechanism: narrow agents solve and check tasks, an external layer preserves the conditions under which their reasoning was valid, and a meta-controller learns when to replay, adapt, synthesize, or seek help. Exact replay and version separation have conditional guarantees. Controlled recurrence reduces repeated logical work, while the complete-family and cross-domain holdouts show that the shared procedure carries task-solving organization into new families with current evidence and typed verification. The model pilot shows how narrow agents communicate through the same contracts. Together, these results establish an initial engineering path toward general intelligence through verified experience, procedural abstraction, and developmental transfer.

\bibliography{references}

\clearpage

\appendix
\section{Implementation protocol and source records}\label{app:protocol}
Clinical cases were generated with seed 1701 and policy cases with seed 20250925, then processed in fixed stream order. Each task supplied a source family, an ordered edition pair, a synthetic fact record, and expected sets of stable rule identifiers. Exact agreement required equality of all five sets: added, removed, modified, unresolved, and current. The full deterministic baseline executed the same encoded rules without memory. The snapshot baseline discarded the old edition. The document-diff baseline discarded individual applicability. The warm comparison primed memory on the first 250 individuals in each domain and updated it online during the later 250. The cold comparison disabled memory on those same later individuals. These are split individuals from familiar families, not held-out rule families.

The complete-family holdout used the same runner with the colorectal screening family withheld in full. Four families supplied 400 training cases. The held-out test contained 100 cases and was evaluated against a frozen train-family memory snapshot. The strategy condition materialized source resolution, factual snapshot alignment, predicate evaluation, delta construction, provenance checking, and typed verdict emission. It did not copy a held-out proof, and the test memory was not updated during evaluation. A case counted as a verified match only when all five typed label sets agreed with the stored scenario package and the source gate passed.

The cross-domain stress runner repeated this protocol for each of the five clinical and three policy families. It evaluated 1,000 new-family cases with frozen train-family memory. For each case it also reversed fact insertion order and added two fields that no represented rule reads. The reported input-invariance result counts equality with the unperturbed fresh evaluation. The source gate and local scenario packages remained unchanged. These measurements quantify procedural transfer, domain isolation, and input invariance in the current implementation.

The logged narrow-agent pilot used eight synthetic cases, four for training and four from new rule families. The direct condition made one call per batch. The typed conditions separated interpretation, reasoning, and cross-checking, and the experience condition added an abstract experience read. Every call record stores its role, model-family label, prompt hash, input hash, output hash, and status. The deployment interface exposed the GPT-6 family label. Rejected malformed outputs remain in the log and are excluded from the pilot score. The pilot provides an initial model-mediated demonstration of the orchestration contract.

The package includes a source manifest file listing the 13 official artifact landing pages and SHA-256 digests of the archived bytes and extracted text used in the study. The official artifacts are not redistributed here. The synthetic case generator and internal case packages are not part of this initial framework archive, so reproducing every reported label requires those additional artifacts. The study evaluates exact replay and controlled procedural transfer within the represented rule environment. It does not establish the complete developmental architecture, autonomous strategy induction, or transfer to unseen task languages.

\section{Minimal operational protocol}
For each task, register the governing source and tool state, normalize facts into typed predicates, and request candidate experiences. Check exact admissibility before replay. Otherwise request a checked adaptation or fresh derivation. The gate validates source identity, scope, logic, uncertainty, and output type. Release only the accepted package. Commit its proof, or record challenge and abstention, then propose abstractions for separate testing. A source or verifier revision challenges every dependent experience before it can support current output.

\end{document}